\documentclass[11pt,a4paper]{article}
\usepackage[hyperref]{emnlp2020}
\usepackage{times}
\usepackage{latexsym}

\usepackage{graphicx}
\usepackage{amsmath}
\usepackage{amsthm}
\usepackage{booktabs}
\usepackage{algorithm}
\usepackage{algorithmic}
\usepackage{amsfonts,amssymb}
\usepackage{microtype}

\aclfinalcopy % Uncomment this line for the final submission
\def\aclpaperid{1008} %  Enter the acl Paper ID here

\title{AutoETER: Automated Entity Type Representation for Knowledge Graph Embedding}

\author{

Guanglin Niu\textsuperscript{\rm 1},
Bo Li\textsuperscript{\rm 1,2},
Yongfei Zhang\textsuperscript{\rm 1,2,3}\thanks{\ \ Corresponding Author}  ,
Shiliang Pu\textsuperscript{\rm 4},
Jingyang Li\textsuperscript{\rm 1}
\\ % All authors must be in the same font size and format. Use \Large and \textbf to achieve this result when breaking a line
\textsuperscript{\rm 1}Beijing Key Laboratory of Digital Media, School of Computer Science and Engineering,
\\Beihang University, Beijing 100191, China\\
\textsuperscript{\rm 2}State Key Laboratory of Virtual Reality Technology and Systems, Beihang University,\\Beijing 100191, China
\textsuperscript{\rm 3}Pengcheng Laboratory, Shenzhen 518055, China\\
\textsuperscript{\rm 4}Hikvision Research Institute, Hangzhou 311500, China
\\ %If you have multiple authors and multiple affiliations
\{beihangngl, boli, yfzhang, lijingyang\}@buaa.edu.cn, pushiliang.hri@hikvision.com% email address must be in roman text type, not monospace or sans serif

}

\date{}

\begin{document}
\maketitle
\begin{abstract}

  Recent advances in Knowledge Graph Embedding (KGE) allow for representing entities and relations in continuous vector spaces. Some traditional KGE models leveraging additional type information can improve the representation of entities which however totally rely on the explicit types or neglect the diverse type representations specific to various relations. Besides, none of the existing methods is capable of inferring all the relation patterns of symmetry, inversion and composition as well as the complex properties of 1-N, N-1 and N-N relations, simultaneously. To explore the type information for any KG, we develop a novel KGE framework with \textbf{\underline{Auto}}mated \textbf{\underline{E}}ntity \textbf{\underline{T}}yp\textbf{\underline{E}} \textbf{\underline{R}}epresentation (AutoETER), which learns the latent type embedding of each entity by regarding each relation as a translation operation between the types of two entities with a relation-aware projection mechanism. Particularly, our designed automated type representation learning mechanism is a pluggable module which can be easily incorporated with any KGE model. Besides, our approach could model and infer all the relation patterns and complex relations. Experiments on four datasets demonstrate the superior performance of our model compared to state-of-the-art baselines on link prediction tasks, and the visualization of type clustering provides clearly the explanation of type embeddings and verifies the effectiveness of our model.
\end{abstract}

\section{Introduction}

In recent years, knowledge graph (KG) has been viewed as a powerful technique for recognition systems and prevalent in many fields such as E-commerce, intelligent healthcare, and public security. Knowledge graphs collect and store a great deal of commonsense or domain knowledge in factual triples composed of entity pairs with their relations. The existing large scale KGs such as Freebase \cite{Bollaker:freebase}, WordNet \cite{Miller:WordNet}, YAGO \cite{YAGO3} have shown their validity in various applications, including question answering \cite{WWW:KBQA}, dialogue generation \cite{He:dialogue} and recommender systems \cite{KGAT19}.

\begin{figure}
	\includegraphics[width=1\columnwidth]{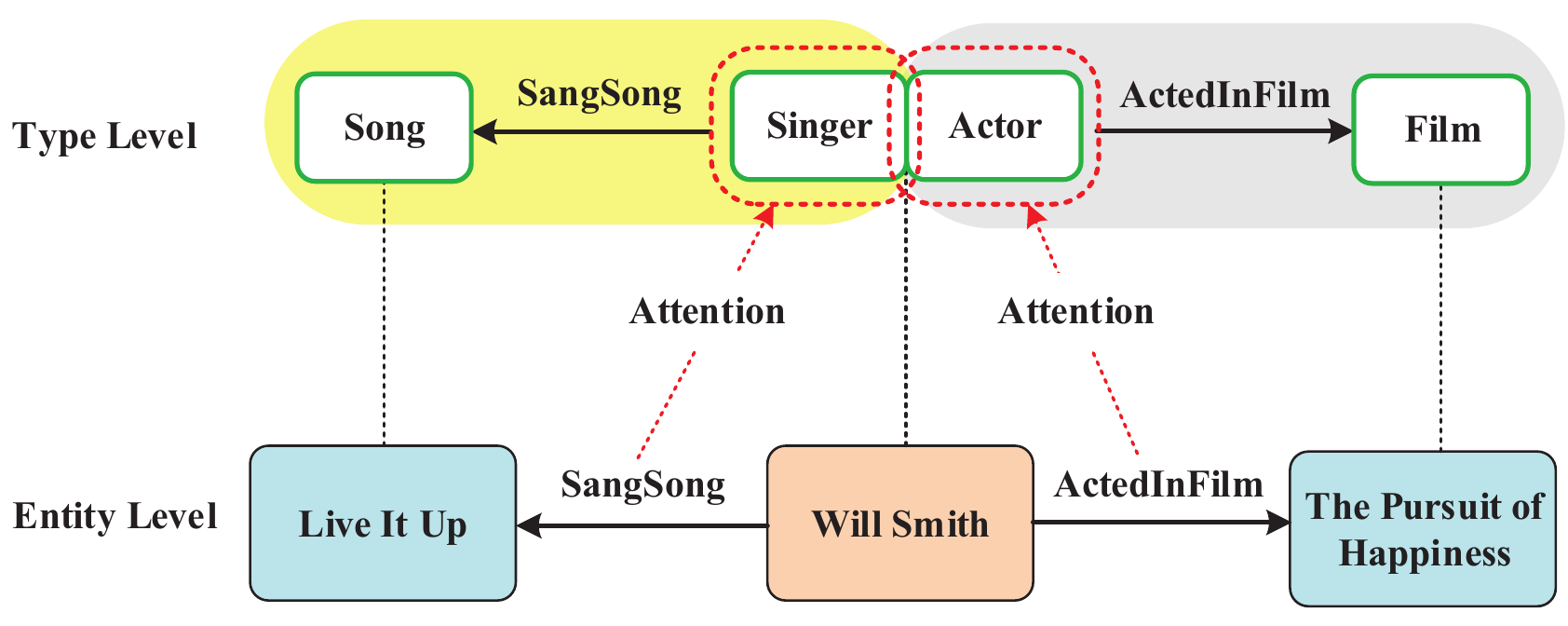}
	\caption{An actual example of the entity-specific triples and the type-specific triples with relation-aware projection mechanism. $Will\ Smith$ has multiple types such as $Singer$ and $Actor$, but only the type $Singer$ should be focused on for the relation $SangSong$.}
	\label{figure1}
\end{figure}

However, the existing KGs are inevitably incomplete whether they are constructed manually or automatically, limiting the effectiveness when exploited for downstream applications. Some existing KG inference approaches such as inductive logic programming algorithm \cite{HAIL}, Markov logic networks based method \cite{pLogicNet} and reinforcement learning-based approach \cite{multi-hop} try to predict entities or relations in KGs but face the limited performance and suffer from the low efficiency. Compared to the above approaches, knowledge graph embedding models could learn the latent representations of the entities and relations and show the best performance on the KG completion task. However, most of the KG embedding models such as TransE \cite{Bordes:TransE} and its variants TransH \cite{Wang:TransH}, TransR \cite{Lin-a:TransR} learn KG embeddings relying on single triples, which simply exploit the structure information implied in KGs.

Entity types define categories of entities that are valid to enhance the representation of entities. In many type-embodied models such as TKRL \cite{Xie:TKRL} and TransT \cite{TransT}, the explicit types are necessary while some KGs (i.e., WordNet) lack them, which limits the versatility of these approaches. JOIE \cite{JOIE} jointly encodes both the ontology and instance views of KGs. Nevertheless, ontologies' concepts always represent the general categories of entities but cannot reflect the specific types, primarily associated with different relations. Jain \cite{Type-sensitive} learned the type embeddings by defining the compatibility between an entity type and a relation. Still, it ignores the semantics implied in a whole triple consisting of a relation jointly with its linked two entity types. Moreover, all the previous type-based approaches neglect the diversity of entity type representations specific to various relations. As Figure \ref{figure1} shows, contrary to the previous researches considering entity types, the triples in the entity level could be extended to triples in the type level. Each entity has multiple types, and diverse types should be focused on for different specific relations.

Additionally, some models embed the entities and relations into the complex vector space instead of the frequently-used real space to improve the capability of representation learning, including ComplEx \cite{Trouillon:ComplEx} and RotatE \cite{RotatE}. Nevertheless, none of the existing embedding models could model and infer all the relation patterns and the complex 1-N, N-1 and N-N relations, simultaneously.

To conduct the KG inference from the perspectives of both entity-specific triples and type-specific triples on any KG, whether the explicit types exist, we propose AutoETER to automatically learn the diverse type representations of each entity when focusing on the various associated relations. Intuitively, the high-dimensional entity embeddings imply the individual features to distinguish the diverse entities. In contrast, the low-dimensional type embeddings capture the general features to discover the similarity of entities according to their categories. Inspired by the translational-based principle in TransE, we expect that given a head entity and its associated relation, the tail entity's type representation can be obtained by $\textbf{type}_{head} + \textbf{relation} = \textbf{type}_{tail}$. Particularly, the latent type embeddings of two head or two tail entities focused on the same relation should be close to each other since they imply the same type. Furthermore, the embeddings of the entity-specific triples and the type-specific triples are capable of modeling and inferring symmetry, inversion, composition, and complex 1-N, N-1, N-N relations.

The contributions of this work are summarized as follows:

\begin{itemize}
\item We model type representations to enrich the general features of entities. A novel model AutoETER is proposed to learn the embeddings of entities, relations and entity types from entity-specific triples and type-specific triples without explicit types in KGs. Furthermore, the type embeddings can be incorporated with the entity embeddings for inference.

\item To the best of our knowledge, we are the first to model and infer all the relation patterns, including symmetry, inversion and composition, as well as complex relations of 1-N, N-1 and N-N for the KG inference.

\item We conduct extensive experiments on link prediction on four real-world benchmark datasets. The evaluation results demonstrate the superiority of our proposed model over other state-of-the-art algorithms. The visualization of clustering type embeddings validates the effectiveness of automatically representing entity types with relation-aware projection.
\end{itemize}

\section{Related Works}

\subsection{Knowledge Graph Inference}

To address the inherent incompleteness of KGs, multiple KG inference methods are investigated and have made significant progress. Traditional researches devote to generate logic rules based on inductive logic programming such as HAIL \cite{HAIL} to predict the missing entities in KGs. However, employing logic rules in KG inference limits generalization performance. Path ranking algorithm (PRA) \cite{Lao:PRA} extracts the relational path features based on random-walk to infer the relationships between entity pairs. DeepPath \cite{multi-hop} is a foundational approach that formulates the multi-hop reasoning as a Markov decision process and leverages reinforcement learning (RL) to find paths in KGs. However, the RL-based multi-hop KG reasoning approaches consume much time in searching paths.

\subsection{KG Embedding Models}

Various KG embedding models have been extensively developed for KG inference in recent years \cite{TKDE:survey}. KGE models are capable of capturing latent representations of entities and relations in KGs independently from hand-crafted rules, and they have shown a strong capacity of efficient computation in many knowledge-aware applications \cite{SurveyKG}. TransE \cite{Bordes:TransE} is the foundational translation-based method, which regards a relation as a translation operation from the head entity to the tail entity. Along with TransE, multiple variants are proposed to improve the embedding performance of KGs \cite{RPJE20,TransGate,ManifoldE}. ConvE \cite{Dettmers:CNN} is a typical method representing entities and relations based on convolutional neural networks (CNN). Another category of KG embedding contains many tensor decomposition models, including DisMult \cite{Yang:ICLR}. Particularly, ComplEx \cite{Trouillon:ComplEx} extends DisMult to learn the KG embeddings in the complex space. RotatE \cite{RotatE} defines a relation as a rotation from source to target entities in a complex space but cannot infer the complex relations 1-N, N-1 and N-N. What's more, all the approaches above purely depend on the triples directly observed in KGs.

\subsection{Models Incorporating Entity Types}

To further improve the performance of KG embedding, various auxiliary information is introduced, such as paths \cite{Lin-b:PTransE,RPJE20}, graph structure \cite{Michael:GNN} and entity types \cite{Xie:TKRL,Krompab:ISWC,TransT}. Among such information, entity types contain less noise and are appropriate for providing more general semantics for each entity. TKRL \cite{Xie:TKRL} projects each entity with the type-specific projection matrices. TransT \cite{TransT} measures the semantic similarity of entities and relations utilizing types. However, all the above type-based KG embedding models require the supervision of explicit types and cannot work on KGs without explicit types. JOIE \cite{JOIE} links entities to their concepts in the ontology for jointly embed the instance-view graph and the ontology-view graph, but the concepts in ontologies provide too broad or even noisy information to represent the specific and precise types of each entity. \cite{Type-sensitive} introduces the compatibility between the embeddings of an entity type and a relation for link prediction. Still, all the existing type-enhanced models neglect that an entity's diverse types should be focused on when this entity is associated with various relations. Meanwhile, the association property implied in the embeddings of the type-specific triples has not been well modeled.

\begin{figure*}
	\centering
	\includegraphics[width=0.99\linewidth]{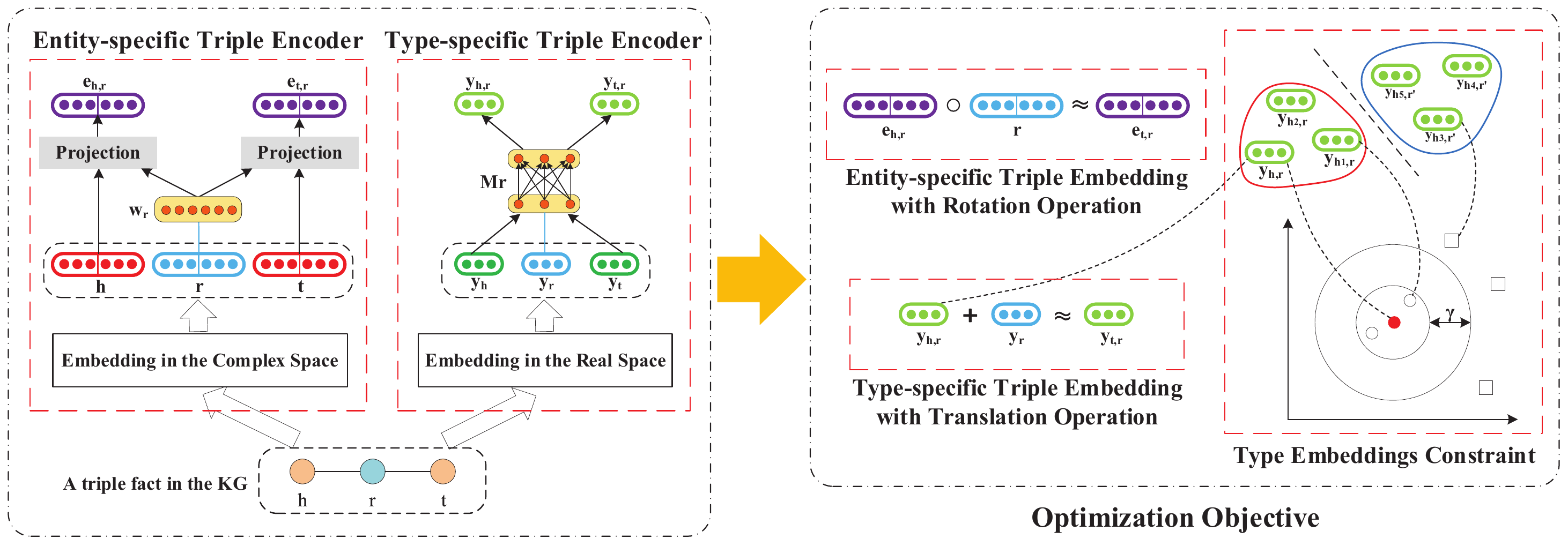}
	\caption{The architecture of AutoETER. Given a triple fact $(h,r,t)$, $\textbf{e}_{h,r}$ and $\textbf{e}_{t,r}$ are the projected entity embeddings in the hyper-plane of relation $r$, $\textbf{y}_{h,r}$ and $\textbf{y}_{t,r}$ are type embeddings focusing on relation $r$. Furthermore, we expect the embeddings of entity-specific triple satisfies rotation operation and type-specific triple satisfies translation operation from head to tail entities. Type embeddings associated with the same relation $r$ are constrained to be closer, where $\gamma$ is the margin enforced between two clusters of type embeddings related to different relations.}
	\label{figure2}
\end{figure*}

\section{AutoETER: KGE with Automated Entity Type Representation}

To cope with the above limitations, we describe the proposed model AutoETER, which aims to automatically learn a variant of type representations semantically compatible with various relations and infer all the relation patterns and complex relations. As figure \ref{figure2} shows, we first embed the entities and relations into complex space via the entity-specific triple encoder with a hyper-plane projection strategy (§\ref{section3.2}). Additionally, the type-specific triple encoder is developed to learn type embeddings incorporated with a relation-aware projection mechanism (§\ref{section3.3}). Meanwhile, the type embeddings are constrained by their similarity derived from the associated relations (§\ref{section3.4}). Afterward, we propose the overall optimization objective with both entity-specific triple and type-specific triple representations and the similarity constraint of the type embeddings (§\ref{section3.5}).

\subsection{Entity-specific Triple Encoder}
\label{section3.2}

We embed the entities and relations into the complex space and regard a relation as the rotation operation from the head entity to the tail entity as in RotatE \cite{RotatE}. To further model and infer the complex relations such as 1-N, N-1 and N-N, we project entities into their associated relation hyper-planes to ensure each entity has various representations concerning the specific relations. In terms of an entity-specific triple $(h, r, t)$, the energy function $E_1(h,r,t)$ is defined as
\begin{align}
    \textbf{e}_{h,r}&=\textbf{h} - \textbf{h}^\top \textbf{w}_r \textbf{h},\ \ \textbf{e}_{t,r}=\textbf{t} - \textbf{t}^\top \textbf{w}_r \textbf{t} \label{eq1} \\
    &E_1(h,r,t)=\Vert \textbf{e}_{h,r} \circ \textbf{r}-\textbf{e}_{t,r} \Vert
    \label{eq2}
\end{align}
where $\textbf{h}\in\mathbb{C}^k$, $\textbf{t}\in\mathbb{C}^k$, $\textbf{r}\in\mathbb{C}^k$ are the embeddings of head entity $h$, tail entity $t$ and relation $r$ in the complex space with dimension $k$. $\textbf{w}_r\in\mathbb{R}^{k}$ denotes the normal vector of the hyper-plane involved in the relation $r$. $\textbf{e}_{h,r}\in\mathbb{C}^{k}$ and $\textbf{e}_{t,r}\in\mathbb{C}^{k}$ represent the entity embeddings of $h$ and $t$ projected in the hyper-plane $\textbf{w}_r$. $\circ$ is the Hadamard product.

On account of the embeddings of entity-specific triples, our model can infer all the relation patterns via the rotation operation from head to tail entities as in RotatE. Particularly, $\textbf{r}$ is constrained to be $|r_i|=1$, $i=1,2,...,k$ for inferring the symmetric relation pattern and at least one element of $\textbf{r}$ is -1 to ensure the diverse representations of head and tail entities. Moreover, the projection operation shown in Eq. \ref{eq1} enables our model to infer the complex relations via various representations of entities regarding different relations.

\subsection{Type-specific Triple Encoder}
\label{section3.3}

Given an entity $e$ and its associated relation $r$ in a triple, we aim to learn the type and relation embeddings with a relation-aware projection mechanism to output the most important information of the type representations:
\begin{equation}
  %f_{att}(T_e,r)=\textbf{ReLU}(\text{M}_r\textbf{y}_e+\textbf{b}_r)
  f_{att}(e,r)=\text{M}_r\textbf{y}_e
  \label{eq3}
\end{equation}
where $\textbf{y}_e\in\mathbb{R}^d$ denotes the type embedding of entity $e$ in the real space with dimension $d$. $\textbf{M}_r\in\mathbb{R}^{d\times d}$ is defined as the projection weight matrix associated with the relation $r$, which could automatically select the latent information of each type embedding most relevant to the relation $r$.

With the relation-aware projection defined in Eq. \ref{eq3}, the energy function involved in type-specific triples is defined as
\begin{equation}
  \begin{split}
   \textbf{y}_{h,r}=f_{att}(h,r),\ \ \textbf{y}_{t,r}=f_{att}(t,r) \\
   E_2(h,r,t)=\Vert \textbf{y}_{h,r} + \textbf{y}_{r} - \textbf{y}_{t,r}\Vert
 \end{split}
   \label{eq4}
\end{equation}
where $\textbf{y}_{h,r}\in \mathbb{R}^d$, $\textbf{y}_{t,r}\in \mathbb{R}^d$ are the type embeddings of entities $h$ and $t$ both focusing on the relation $r$ and $\textbf{y}_{r}\in \mathbb{R}^d$ denotes the embedding of the relation $r$ in the type-specific triple. In terms of the energy function in Eq. \ref{eq4}, we expect that
\begin{equation}
  \textbf{y}_{h,r} + \textbf{y}_{r} = \textbf{y}_{t,r}
  \label{eq5}
\end{equation}

Furthermore, with the type and relation embeddings learned in the real spaces, our model cost fewer parameters and could model and infer all the relation patterns including symmetry (Lemma \ref{lemma1}), inversion (Lemma \ref{lemma2}) and composition (Lemma \ref{lemma3}) as well as the complex properties of relations:
\newtheorem{lemma}{Lemma}
\begin{lemma} \label{lemma1}
Our model could infer relation pattern of symmetry by type-specific triple embeddings.
\end{lemma}

\begin{proof}
If a relation $r$ is symmetric, two triples $(h,r,t)$ and $(t,r,h)$ will hold. From Eq. \ref{eq5}, the correlations among the embeddings of types and relations can be obtained as:
\begin{align}
  \textbf{y}_{h,r} + \textbf{y}_{r} = \textbf{y}_{t,r},\ \ \textbf{y}_{t,r} + \textbf{y}_{r} = \textbf{y}_{h,r}\label{eq6}
\end{align}
From Eq. \ref{eq6}, we can further derive that
\begin{equation}
  \textbf{y}_{h,r}=\textbf{y}_{t,r},\ \ \textbf{y}_{r}=\textbf{0}
  \label{eq8}
\end{equation}
We prove that the embedding of a symmetric relation should be zero vector, and the type embeddings of head and tail entities should be equal. The above results are reasonable owing to the focused types of two entities linked by the symmetric relation are supposed to be same.
\end{proof}

\begin{lemma} \label{lemma2}
Our model is able to infer relation pattern of inversion by type-specific triple embeddings.
\end{lemma}

\begin{proof}
With the inverse relations $r_1$ and $r_2$, two triples $(h,r_1,t)$ and $(t,r_2,h)$ hold. From Eqs. \ref{eq3}, \ref{eq4} and \ref{eq5}, it can be retrieved that
\begin{align}
  \textbf{M}_{r_1}\textbf{y}_{h} + \textbf{y}_{r_1} &= \textbf{M}_{r_1}\textbf{y}_{t} \label{eq11} \\
  \textbf{M}_{r_2}\textbf{y}_{t} + \textbf{y}_{r_2} &= \textbf{M}_{r_2}\textbf{y}_{h} \label{eq12}
\end{align}
We can define a transform matrix $\textbf{P}\in \mathbb{R}^{d\times d}$ that satisfies
\begin{equation}
  \textbf{M}_{r_1}=\textbf{P}\textbf{M}_{r_2}
  \label{eq13}
\end{equation}
Substituting Eq. \ref{eq13} into Eq. \ref{eq12}, the latter can be modified as
\begin{equation}
  \textbf{M}_{r_1}\textbf{y}_{t} + \textbf{P}\textbf{y}_{r_2} = \textbf{M}_{r_1}\textbf{y}_{h}
  \label{eq14}
\end{equation}
Then, substituting Eq. \ref{eq14} into Eq. \ref{eq11}, it yields that
\begin{equation}
  \textbf{y}_{r_1} = -\textbf{P}\textbf{y}_{r_2}
  \label{eq16}
\end{equation}
We can model and infer the inverse relations with the relation embeddings satisfying the relationship as in Eq. \ref{eq16}.
\end{proof}

\begin{lemma} \label{lemma3}
Our model is capable of inferring the relations of composition by type-specific triple embeddings.
\end{lemma}

\begin{proof}
  On account of the relations of composition pattern $r_3(a,c)\Leftarrow r_1(a,b) \wedge r_2(b,c)$, the corresponding triples $(a,r_1,b)$, $(b,r_2,c)$ and $(a,r_3,c)$ hold. Meanwhile, considering Eqs. \ref{eq3}, \ref{eq4} and \ref{eq5}, it can be obtained that
  \begin{align}
    \textbf{M}_{r_1}\textbf{y}_{a} + \textbf{y}_{r_1} = \textbf{M}_{r_1}\textbf{y}_{b} \label{eq17} \\
    \textbf{M}_{r_2}\textbf{y}_{b} + \textbf{y}_{r_2} = \textbf{M}_{r_2}\textbf{y}_{c} \label{eq18} \\
    \textbf{M}_{r_3}\textbf{y}_{a} + \textbf{y}_{r_3} = \textbf{M}_{r_3}\textbf{y}_{c} \label{eq19}
  \end{align}
We can define two transform matrices $\textbf{P}\in \mathbb{R}^{d\times d}$ and $\textbf{Q}\in \mathbb{R}^{d\times d}$ to satisfy
\begin{align}
  \textbf{P}\textbf{M}_{r_1}=\textbf{M}_{r_3} \label{eq20} \\
  \textbf{Q}\textbf{M}_{r_2}=\textbf{M}_{r_3} \label{eq21}
\end{align}
Substituting Eq. \ref{eq20} into Eq. \ref{eq17} and Eq. \ref{eq21} into Eq. \ref{eq18}, respectively, we can derive that
\begin{align}
  \textbf{M}_{r_3}\textbf{y}_{a} + \textbf{P}\textbf{y}_{r_1} = \textbf{M}_{r_3}\textbf{y}_{b} \label{eq22} \\
  \textbf{M}_{r_3}\textbf{y}_{b} + \textbf{Q}\textbf{y}_{r_2} = \textbf{M}_{r_3}\textbf{y}_{c} \label{eq23}
\end{align}
Substituting Eq. \ref{eq22} into Eq. \ref{eq23}, it can be retrieved that
\begin{equation}
  \textbf{M}_{r_3}\textbf{y}_{a} + \textbf{P}\textbf{y}_{r_1} + \textbf{Q}\textbf{y}_{r_2} = \textbf{M}_{r_3}\textbf{y}_{c}
  \label{eq24}
\end{equation}
Combining Eqs. \ref{eq19} and \ref{eq24}, we can model the correlation among the relation embeddings of composition pattern as
\begin{equation}
  \textbf{y}_{r_3} = \textbf{P}\textbf{y}_{r_1} + \textbf{Q}\textbf{y}_{r_2}
  \label{eq25}
\end{equation}
We prove that we can model and infer the relations of composition pattern for type-specific triples with the relation embeddings as shown in Eq. \ref{eq25}.
\end{proof}

Specific to the inference on type-specific triples with the relations of complex properties 1-N, N-1 and N-N, we could exploit the various representations of an entity type associated with different relations via the relation-aware projection mechanism defined in Eq. \ref{eq3} to infer on these relations.

\subsection{Type Embeddings Similarity Constraint}
\label{section3.4}

In addition to learning type embeddings by the type-specific triple encoder (§\ref{section3.3}), the type representations should be constrained by the similarity between the entity types. The type embeddings of head entities involved in the triples with the same relation are closer to each other (the same as type embeddings of tail entities). Thus, as for two triples with the same relation, we expect that
\begin{equation}
  \textbf{y}_{h_1,r} = \textbf{y}_{h_2,r}, \ \ \textbf{y}_{t_1,r} = \textbf{y}_{t_2,r}
  \label{eq26}
\end{equation}
where $\textbf{y}_{h_1,r}$ and $\textbf{y}_{h_2,r}$ are type embeddings of head entities while $\textbf{y}_{t_1,r}$ and $\textbf{y}_{t_2,r}$ are type embeddings of tail entities. Particularly, they all focus on the same relation $r$ by the relation-aware projection mechanism of Eq. \ref{eq3}.

Now, considering any two triples $(h_1, r_1, t_1)$ and $(h_2, r_2, t_2)$, we design the energy function for evaluating the dissimilarity of the type embeddings as
\begin{align}
  E_3((h_1, r_1, t_1), (h_2, r_2, t_2)) = &\frac{1}{2}\big(\Vert \textbf{y}_{h_1,r_1} - \textbf{y}_{h_2,r_2} \Vert \nonumber\\
  &+ \Vert \textbf{y}_{t_1,r_1} - \textbf{y}_{t_2,r_2} \Vert\big) \label{eq27}
\end{align}
where $\textbf{y}_{h_1,r_1}$ and $\textbf{y}_{h_2,r_2}$ are two head entity type embeddings, $\textbf{y}_{t_1,r_1}$ and $\textbf{y}_{t_2,r_2}$ are two tail entity type embeddings, and they are all associated with the relation $r_1$ or $r_2$. Therefore, we expect the value derived from Eq. \ref{eq27} tends to be smaller if $r_1$ and $r_2$ are the same relation.

\subsection{Optimization Objective}
\label{section3.5}

The designed entity-specific triples encoder, type-specific triples encoder and type representations similarity constraint could be trained as a unified end-to-end model. We optimize our model according to a three-component objective function:
\begin{equation}
  L = \sum_{(h,r,t)\in S}{\left\{\sum_{(h',r,t')\in S'}{\Big\{L_1 + \alpha_1 L_2\Big\}} + \alpha_2 L_3\right\}}
  \label{eq29}
\end{equation}
in which the overall training objective consists of three components: $L_1$ and $L_2$ are two pair-wise loss functions that correspond to the entity-specific triple encoder and the type-specific triple encoder, respectively, and $L_3$ is a triple loss function for constraining the type embeddings. $\alpha_1$ and $\alpha_2$ denote the weights of $L_2$ and $L_3$ for the tradeoff between the entity-specific triple, the type-specific triple and the type similarity constraint. $S$ contains all the triples in the train set, and $S'$ is the negative sample set generated by replacing the entities in $S$. Specifically, $L_1$, $L_2$ and $L_3$ are defined as
  \begin{align}
    &L_1 = -\log\sigma(\gamma_1-E_1(h,r,t)) \nonumber\\
    & \ \ \ \ \ \ \ \ \ - \log\sigma(E_1(h',r,t') - \gamma_1) \label{eq30} \\
    &L_2 = \max\big[0, E_2(h, r, t) + \gamma_2 - E_2(h', r, t')\big] \label{eq31} \\
    &L_3 = \sum_{(hp,r,tp)\in Y}{\sum_{(hn,r',tn)\in Y'}\max\big[0, E_3((h, r, t),} \nonumber \\
    & \ \ \ \ \ \ \ \ (hp, r, tp))+\gamma_3- E_3((h, r, t), (hn, r', tn))\big] \label{eq32}
  \end{align}
where $\gamma_1$, $\gamma_2$ and $\gamma_3$ denote the fixed margins in $L_1$, $L_2$ and $L_3$, respectively. In specific, $L_3$ can be viewed as the regularization in optimization for restraining the entity type embeddings. $\sigma$ denotes the sigmoid function. max[0,x] is the function to select the larger value between 0 and x. Particularly, in Eq. \ref{eq32}, the triple $(h, r, t)$ is regarded as the anchor instance and $(h_p, r, t_p)$ is a positive instance in the set $Y$ containing other triples correlated to the same relation $r$, while $(h_n, r', t_n)$ is any negative instance in the set $Y'$ containing the other triples without the relation $r$. Besides, we employ self-adversarial sampling as in \cite{RotatE}.

\section{Experiment Results}

In this section, we evaluate our model AutoETER for KG completion on four real-world benchmark datasets. Additionally, we visualize the clustering results of type embeddings for demonstrating the effectiveness of representing types automatically.

\subsection{Experimental Setup}

\subsubsection{Datasets}

We utilize four standard datasets$\footnote{Datasets could be found at onedrive: \url{https://1drv.ms/u/s!Ajh_jEjaTE0SbbceogcmdwSu9ME?e=zfw6sN}}$ for link prediction tasks: FB15K \cite{Bordes:TransE} is a widely used dataset that is a subgraph of the commonsense knowledge graph Freebase. WN18 \cite{Bordes:TransE} is a subset of the lexical knowledge graph WordNet. YAGO3-10 \cite{Dettmers:CNN} is a subset of YAGO. Each of the three datasets consists of all the relation patterns, including symmetry, inversion, composition and complex 1-N, N-1 and N-N of relations. FB15K-237 \cite{FB15K237} is a subset of FB15K and removes all the inverse relations. Table \ref{table1} exhibits the statistics of all the datasets exploited.
\begin{table}\small
\centering
\renewcommand\tabcolsep{3.0pt}
\begin{tabular}{ccccc}
\toprule
Dataset		& WN18		& YAGO3-10	& FB15K	  & FB15K-237\\
\midrule
\#Entity		& 40,943		& 123,182  & 14,951	& 14,505\\
\#Relation	& 18		    & 37	    & 1,345	  & 237\\
\#Train     & 141,442   & 1,079,040  & 483,142 & 272,115\\
\#Valid     & 5,000     & 5,000   & 50,000  & 17,535\\
\#Test      & 5,000     & 5,000  & 59,071  & 20,466\\
\bottomrule
\end{tabular}
\caption{Statistics of datasets used in the experiments.}
\label{table1}
\end{table}

\subsubsection{Evaluation Protocol}

The link prediction task aims to predict when the head or tail entity of a triple in the test set is missing. For link prediction, all the entities in the KG are respectively replaced with the missing entity to generate the candidate triples. Then, on account of each candidate triple $(h,r,t)$, we combine the two perspectives of the entity-specific triple jointly with the type-specific triple to evaluate the plausibility of this candidate triple, and the energy function for evaluation is designed as follows:
\begin{equation}
	E_{pred}(h,r,t) = E_1(h,r,t)+\alpha_1 E_2(h,r,t) \label{eq33}
\end{equation}

The above energy function $E_{pred}(h,r,t)$ is composed of the energy functions $E_1(h,r,t)$ (with regard to the entity-specific triple) and $E_2(h,r,t)$ (with respect to the type-specific triple) defined in Eqs. \ref{eq2} and \ref{eq4}, respectively. $\alpha_1$ is the weight which is the same as in Eq. \ref{eq29} for a trade-off. Then, the scores with respect to all the candidate triples are calculated by Eq. \ref{eq33}. Subsequently, these scores are sorted in ascending order, and further, the correct triple rank can be obtained.

\begin{table*}\scriptsize
\centering
\begin{tabular}{cccccccccccc}
\toprule
    Model & \multicolumn{5}{c}{FB15K}  &  & \multicolumn{5}{c}{WN18}\\
    \cline{2-6} \cline{8-12}
                           & MR      & MRR       & Hits@1       & Hits@3       & Hits@10   & & MR     & MRR      & Hits@1    & Hits@3    & Hits@10 \\
\midrule
TransE \cite{Bordes:TransE}     & -		    & 0.463     & 0.297     & 0.578     & 0.749  & & -      & 0.495    & 0.113  & 0.888  & 0.943 \\
DistMult \cite{Yang:ICLR}       & 42  		& 0.798     & -         & -         & 0.893  & & 655    & 0.797    & -      & -      & 0.946 \\
HolE \cite{HolE}                & -   		& 0.524     & 0.402     & 0.613     & 0.739  & & -      & 0.938    & 0.930  & 0.945  & 0.947 \\
ComplEx \cite{Trouillon:ComplEx}& -   	  & 0.692     & 0.599     & 0.759     & 0.840  & & -      & 0.941    & 0.936  & 0.945  & 0.947 \\
ConvE \cite{Dettmers:CNN}       & 51		  & 0.657     & 0.558     & 0.723     & 0.831  & & 374    & 0.943    & 0.935  & 0.946  & 0.956   \\
RotatE \cite{RotatE}            & \underline{40}		  & \underline{0.797}     & \underline{0.746}     & \underline{0.830}     & \underline{0.884}  & & \underline{309}    & 0.949    & \underline{0.944}  & 0.952   & 0.959  \\
QuatE \cite{QuatE}            & \underline{40}		  & 0.765     & 0.692     & 0.819     & 0.878  & & 393    & \underline{0.950}    & 0.942  & \textbf{0.954}   & 0.959  \\
\midrule
R-GCN \cite{Michael:GNN}        & - 		  & 0.696     & 0.601     & 0.760     & 0.852  & & -      & 0.819    & 0.697  & 0.929  & \textbf{0.964}   \\
PTransE \cite{Lin-b:PTransE}    & 54	    & 0.679     & 0.565     & 0.768     & 0.855  & & 472      & 0.890    & 0.931  & 0.942   & 0.945  \\
TKRL \cite{Xie:TKRL}            & 68	    & -     & -     & -     & 0.694  & & -      & -    & -  & -   & -  \\
TypeComplex \cite{Type-sensitive}    & -	    & 0.753     & 0.677     & -     & 0.869  & & -      & 0.939    & 0.932  & -   & 0.951  \\
\midrule
AutoETER	        & \textbf{33}	& \textbf{0.799}     & \textbf{0.750}  & \textbf{0.833}   & \textbf{0.896}	& & \textbf{174}  & \textbf{0.951}  & \textbf{0.946}   & \textbf{0.954}	& \underline{0.961}  \\
\bottomrule
\end{tabular}
\caption{Evaluation Results on FB15K and WN18. Best results are in \textbf{bold} and second best results are \underline{underlined}.}
\label{table2}
\end{table*}

\begin{table*}\scriptsize
\centering
\begin{tabular}{cccccccccccc}
\toprule
    Model & \multicolumn{5}{c}{FB15K-237}  &  & \multicolumn{5}{c}{YAGO3-10}\\
    \cline{2-6} \cline{8-12}
                           & MR      & MRR       & Hits@1    & Hits@3    & Hits@10   & & MR     & MRR      & Hits@1    & Hits@3    & Hits@10 \\
\midrule
TransE \cite{Bordes:TransE}     & 357		  & 0.294     & -         & -         & 0.465  & & -    & -      & -      &   & - \\
DistMult \cite{Yang:ICLR}       & 254  		& 0.241     & 0.155     & 0.263     & 0.419  & & 5926 & 0.34    & 0.24    & 0.38   & 0.54 \\
ComplEx \cite{Trouillon:ComplEx}& 339   	& 0.247     & 0.158     & 0.275     & 0.428  & & 6531   & 0.36    & 0.26  & 0.40  & 0.55 \\
ConvE \cite{Dettmers:CNN}       & 244		  & 0.325     & 0.237     & 0.356     & 0.501  & & \underline{1671}    & 0.44    & 0.35  & 0.49  & 0.62   \\
RotatE \cite{RotatE}            & 177		  & 0.338     & \underline{0.241}     & \underline{0.375}     & \underline{0.533}  & & 1767    & \underline{0.495}    & \underline{0.402}  & \underline{0.550}   & \underline{0.670}  \\
QuatE \cite{QuatE}            & \underline{172}		  & 0.311     & 0.220     & 0.344     & 0.495  & & -    & -    & -  & -   & -  \\
\midrule
R-GCN \cite{Michael:GNN}        & - 		  & 0.249     & 0.151     & 0.264     & 0.417  & & -       & -    & -  & -  & -   \\
PTransE \cite{Lin-b:PTransE}    & 302		  & \textbf{0.363}     & 0.234     & 0.374     & 0.526  && -    & -  & -  & -   \\
TypeComplex \cite{Type-sensitive}    & -	    & 0.259     & 0.186     & -     & 0.411  & & -      & 0.411    & 0.319  & -   & 0.609  \\
\midrule
AutoETER	 & \textbf{170}	& \underline{0.344}     & \textbf{0.250}  & \textbf{0.382}   & \textbf{0.538}	& & \textbf{1179}  & \textbf{0.550}  & \textbf{0.465}   & \textbf{0.605}	& \textbf{0.699}  \\
\bottomrule
\end{tabular}
\caption{Evaluation Results on FB15K-237 and YAGO3-10 datasets.}
\label{table3}
\end{table*}

Three standard metrics are employed to evaluate the performance of link prediction:\\
1) Mean Rank (MR) of the correct triples.\\
2) Mean Reciprocal Rank (MRR) of the correct triples.\\
3) Hits@n measures the proportion of the correct triples in top-n candidate triples.

We also follow the filtered setting as the previous study \cite{Dettmers:CNN} that evaluates the performance by filtering out the corrupt triples already exist in the KG.

\subsubsection{Baselines and Hyper-parameters}

We compare the developed model AutoETER with two categories of the state-of-the-art baselines: (1) Models only considering entity-specific triples including TransE, DisMult, HolE, ComplEx, ConvE, RotatE and QuatE; (2) Models introducing additional information such as TKRL with explicit types and the type-sensitive model TypeComplex, R-GCN with graph structure and PTransE with paths. All the baselines are selected because they achieve good performance and provide source codes for ensuring the reliability and reproducibility of the results. The results of R-GCN are from \cite{QuatE}. The results of TKRL are from \cite{Xie:TKRL}. The results of PTransE$\footnote{\url{https://github.com/thunlp/KB2E/tree/master/PTransE}}$, TypeComplex$\footnote{\url{https://github.com/dair-iitd/KBI/tree/master/kbi-pytorch}}$ and QuatE $\footnote{\url{https://github.com/cheungdaven/QuatE}}$ are obtained by using their source codes. The other results of the baselines are from \cite{RotatE}.

We tune our model utilizing a grid search to select the optimal hyper-parameters. The optimal configurations are provided as: the batch size is set as 1024, the learning rate is $lr=0.0001$, and the weights in optimization are $\alpha_1=0.1, \alpha_2=0.5$. The dimension of the entity and relation embeddings in entity-specific triples is $k=1000$, the dimension of the type and relation embeddings in type-specific triples is $d=200$. For datasets FB15K and YAGO3-10, the three fixed margins are set as $\gamma_1=22$, $\gamma_2=8$, $\gamma_3=6$. For datasets WN18 and FB15K-237, $\gamma_1=10$, $\gamma_2=6$, $\gamma_3=3$.

\begin{figure*}
  \centering
	\includegraphics[scale=0.52]{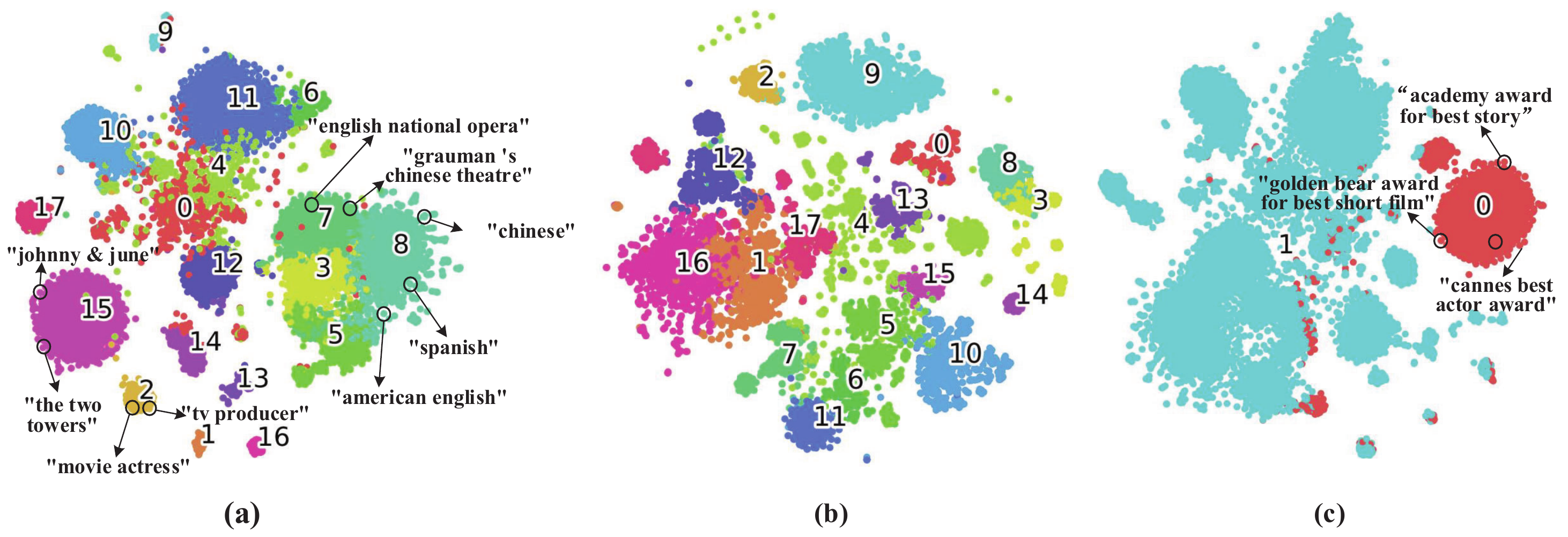}
	\caption{The visualization of type embeddings clustering on FB15K-237. (a) The clustering of the original type embeddings. (b) The clustering of the entity embeddings. (c) The clustering of the type embeddings all focusing on the relation $/award/award\_category/nominated\_for$.}
	\label{figure3}
\end{figure*}

\subsection{Evaluation Results and Analyses}

\begin{table*}\small
\centering
\begin{tabular}{cccccccccc}
\toprule
	Model  & \multicolumn{4}{c}{Head Entity Prediction (Hits@10)}	& & \multicolumn{4}{c}{Tail Entity Prediction (Hits@10)}\\
\cline{2-5} \cline{7-10}
	 & 1-1	& 1-N	& N-1	& N-N	& & 1-1	& 1-N	& N-1	& N-N\\
\midrule
TransE \cite{Bordes:TransE}	& 0.437	& 0.657	& 0.182	& 0.472	& & 0.437	& 0.197	& 0.667	& 0.500\\
TransH \cite{Wang:TransH}   & 0.668	& 0.876	& 0.287	& 0.645	& & 0.655	& 0.398	& 0.833	& 0.672\\
TransR	\cite{Lin-a:TransR} & 0.788	& 0.892	& 0.341	& 0.692	& & 0.792	& 0.374	& 0.904	& 0.721\\
RotatE \cite{RotatE}	  & \underline{0.922}	& \underline{0.967}	& 0.602	& \underline{0.893}	& & \underline{0.923}	& 0.713	& \underline{0.961}	& \underline{0.922}\\
\midrule
PTransE \cite{Lin-b:PTransE}	  & 0.910	& 0.928	& \underline{0.609}	& 0.838	& & 0.912	& \textbf{0.740}	& 0.889	& 0.864\\
\midrule
AutoETER	& \textbf{0.933}	& \textbf{0.979}	& \textbf{0.618}	& \textbf{0.903}	& & \textbf{0.931}	& \underline{0.717}	& \textbf{0.968}	& \textbf{0.927}\\
\bottomrule
\end{tabular}
\caption{Evaluation results on FB15K by mapping properties of relations.}
\label{table4}
\end{table*}

Table \ref{table2} and Table \ref{table3} report the evaluation results of link prediction on the four datasets. We can observe that our model AutoETER outperforms all the baselines, including the state-of-the-art models RotatE and QuatE. These results demonstrate the superiority of modeling and inferring all the relation patterns and the complex relations by our model. Specifically, AutoETER performs better than the type-embodied models TKRL and TypeComplex, emphasizing the type representations learned automatically with relation-aware projection by AutoETER are more effective for inference than totally leveraging the explicit types or ignoring the diversity of type embeddings focusing on various relations. Furthermore, AutoETER outperforms RotatE because AutoETER could infer the complex relations of 1-N, N-1 as well as N-N and takes advantage of type representations. These results all illustrate the type representations learned from KGs are available to predict entities more accurately by restricting the candidate entities with type embeddings.

In view of more diverse relations existed in FB15K compared with the other three datasets, we select FB15K to evaluate link prediction performance by mapping 1-1, 1-N, N-1 and N-N relations. The results are shown in Table \ref{table4}. Our model achieves better performance on both head entity prediction and tail entity prediction than other baselines particularly RotatE, which illustrates the superiority of capturing various representations of entities specific to different relations with the relation-aware projection mechanism to represent entity types.

\subsection{Ablation Study}

We conduct the ablation study of our model on dataset FB15K-237 when we only omit the type similarity constraint (-TSC) and omit the type representation (-TR) from our model. Table \ref{table5} demonstrates that our model performs better than the two ablated models. It illustrates the type representation and the type similarity constraint both significantly impact the performance of link prediction and suggests that our automatically learned type representations play a pivotal role in our approach.

\begin{table}\small
\centering
\begin{tabular}{c|ccccc}
\toprule
    Model                       & MR      & MRR       & H@1    & H@3    & H@10   \\
\midrule
AutoETER	                      & \textbf{170}	    & \textbf{0.344}     & \textbf{0.250}  & \textbf{0.382}   & \textbf{0.538}	\\
-TSC                            & 175		  & 0.342     & 0.246  & 0.379   & 0.536  \\
-TR                             & 177		  & 0.340     & 0.244  & 0.377   & 0.534  \\
\bottomrule
\end{tabular}
\caption{Ablation study on FB15K-237. ``H@'' is the abbreviation of ``Hits@''.}
\label{table5}
\end{table}

\subsection{Visualization of Clustering Entity Type Representations}

We utilize Kmeans to cluster the type embeddings and further employ t-SNE to implement dimensionality reduction for 2d visualization. As Figure \ref{figure3}(a) shows, some type embeddings are clustered into independent categories, while some clusters stay close to each other because these entities share many common types. For instance, $johnny\&june$ and $the\ two\ towers$ are clustered into the same category which actually represents the type $movie$ as we know. Figure \ref{figure3}(b) shows the clustering of the entity embeddings. It can be clearly observed that entity type clustering has better compactness than entity clustering, which demonstrates that entity type embeddings could reflect the characteristics of types. The type embeddings focusing on relation $/award/award\_category/nominated\_for$ are visualized in Figure \ref{figure3}(c). It is evident that some type embeddings representing the type $award$ such as $academy\ award\ for\ best\ story$ and $cannes\ best\ actor\ award$ are clustered into the same category while others stay far away. These visualization results explain the effectiveness of our type embeddings learned automatically with relation-aware projection from the KG.

\section{Conclusion and Future Work}

In this paper, we propose an AutoETER framework to learn type representations for enriching KG embedding automatically. We introduce two classes of encoders to learn the entity-specific triple and type-specific triple embeddings, which could model and infer all the relation patterns of symmetry, inversion and composition as well as the complex 1-N, N-1 and N-N relations. We also constrain the type embeddings by the type similarity. Our experiments on four real-world datasets for link prediction illustrate the superiority of our model and the visualization of the type embeddings clustering verifies the availability of representing types automatically. In future work, we intend to extend our approach to obtain the better type representations incorporating the supervision of ontologies.

\section*{Acknowledgments}

This work was partially supported by the National Natural Science Foundation of China (No. 61772054, 62072022), and the NSFC Key Project (No. 61632001) and the Fundamental Research Funds for the Central Universities.

\bibliography{anthology,emnlp2020}

\begin{thebibliography}{32}
\expandafter\ifx\csname natexlab\endcsname\relax\def\natexlab#1{#1}\fi

\bibitem[{Bollacker et~al.(2008)Bollacker, Evans, Paritosh, and
  Sturge}]{Bollaker:freebase}
Kurt Bollacker, Colin Evans, Praveen Paritosh, and Tim Sturge. 2008.
\newblock Freebase: A collaboratively created graph database for structuring
  human knowledge.
\newblock In \emph{SIGMOD}, page 1247–1250.

\bibitem[{Bordes et~al.(2013)Bordes, Usunier, Garcia-Duran, Weston, and
  Yakhnenko}]{Bordes:TransE}
Antoine Bordes, Nicolas Usunier, Alberto Garcia-Duran, Jason Weston, and Oksana
  Yakhnenko. 2013.
\newblock Translating embeddings for modeling multi-relational data.
\newblock In \emph{NIPS}, page 2787–2795.

\bibitem[{Dettmers et~al.(2018)Dettmers, Minervini, Stenetorp, and
  Riedel}]{Dettmers:CNN}
Tim Dettmers, Pasquale Minervini, Pontus Stenetorp, and Sebastian Riedel. 2018.
\newblock Convolutional {2D} knowledge graph embeddings.
\newblock In \emph{AAAI}, page 1811–1818.

\bibitem[{Diefenbach et~al.(2018)Diefenbach, Singh, and Maret}]{WWW:KBQA}
Dennis Diefenbach, Kamal Singh, and Pierre Maret. 2018.
\newblock Wdaqua-core1: a question answering service for rdf knowledge bases.
\newblock In \emph{WWW}, pages 1087--1091.

\bibitem[{Hao et~al.(2019)Hao, Chen, Yu, Sun, and Wang}]{JOIE}
Junheng Hao, Muhao Chen, Wenchao Yu, Yizhou Sun, and Wei Wang. 2019.
\newblock Universal representation learning of knowledge bases by jointly
  embedding instances and ontological concepts.
\newblock In \emph{KDD}, pages 1709--1719.

\bibitem[{He et~al.(2017)He, Balakrishnan, Eric, and Liang}]{He:dialogue}
He~He, Anusha Balakrishnan, Mihail Eric, and Percy Liang. 2017.
\newblock Learning symmetric collaborative dialogue agents with dynamic
  knowledge graph embeddings.
\newblock In \emph{ACL}, pages 1766--1776.

\bibitem[{Jain et~al.(2018)Jain, Kumar, Mausam1, and
  Chakrabarti}]{Type-sensitive}
Prachi Jain, Pankaj Kumar, Mausam1, and Soumen Chakrabarti. 2018.
\newblock Type-sensitive knowledge base inference without explicit type
  supervision.
\newblock In \emph{ACL}, pages 75--80.

\bibitem[{Ji et~al.(2020)Ji, Pan, Cambria, Marttinen, and Yu}]{SurveyKG}
Shaoxiong Ji, Shirui Pan, Erik Cambria, Pekka Marttinen, and Philip~S. Yu.
  2020.
\newblock A survey on knowledge graphs: Representation, acquisition and
  applications.
\newblock \emph{arXiv preprint arXiv:2002.00388}.

\bibitem[{Krompaß et~al.(2015)Krompaß, Baier, and Tresp}]{Krompab:ISWC}
Denis Krompaß, Stephan Baier, and Volker Tresp. 2015.
\newblock Type-constrained representation learning in knowledge graphs.
\newblock In \emph{ISWC}, page 640–655.

\bibitem[{Lao et~al.(2011)Lao, Mitchell, and Cohen}]{Lao:PRA}
Ni~Lao, Tom Mitchell, and William~W. Cohen. 2011.
\newblock Random walk inference and learning in a large scale knowledge base.
\newblock In \emph{EMNLP}, pages 529--539.

\bibitem[{Lin et~al.(2018)Lin, Socher, and Xiong}]{multi-hop}
Xi~Victoria Lin, Richard Socher, and Caiming Xiong. 2018.
\newblock Multi-hop knowledge graph reasoning with reward shaping.
\newblock In \emph{EMNLP}, page 3243–3253.

\bibitem[{Lin et~al.(2015{\natexlab{a}})Lin, Liu, Luan, Sun, Rao, and
  Liu}]{Lin-b:PTransE}
Yankai Lin, Zhiyuan Liu, Huanbo Luan, Maosong Sun, Siwei Rao, and Song Liu.
  2015{\natexlab{a}}.
\newblock Modeling relation paths for representation learning of knowledge
  bases.
\newblock In \emph{EMNLP}, pages 705--714.

\bibitem[{Lin et~al.(2015{\natexlab{b}})Lin, Liu, Sun, Liu, and
  Zhu}]{Lin-a:TransR}
Yankai Lin, Zhiyuan Liu, Maosong Sun, Yang Liu, and Xuan Zhu.
  2015{\natexlab{b}}.
\newblock Learning entity and relation embeddings for knowledge graph
  completion.
\newblock In \emph{AAAI 2015}, page 2181–2187.

\bibitem[{Ma et~al.(2017)Ma, Ding, Jia, Wang, and Guo}]{TransT}
Shiheng Ma, Jianhui Ding, Weijia Jia, Kun Wang, and Minyi Guo. 2017.
\newblock {TransT}: Type-based multiple embedding representations for knowledge
  graph completion.
\newblock In \emph{ECML PKDD}, pages 717--733.

\bibitem[{Michael et~al.(2018)Michael, N., Peter, van~den Berg~Rianne, Ivan,
  and Max}]{Michael:GNN}
Schlichtkrull Michael, Kipf~Thomas N., Bloem Peter, van~den Berg~Rianne, Titov
  Ivan, and Welling Max. 2018.
\newblock Modeling relational data with graph convolutional networks.
\newblock In \emph{ESWC}.

\bibitem[{Miller(1995)}]{Miller:WordNet}
George~A. Miller. 1995.
\newblock Wordnet: A lexical database for english.
\newblock \emph{Communications of the ACM}, 38(11):39--41.

\bibitem[{Nickel et~al.(2016)Nickel, Rosasco, and Poggio}]{HolE}
Maximilian Nickel, Lorenzo Rosasco, and Tomaso~A Poggio. 2016.
\newblock Holographic embeddings of knowledge graphs.
\newblock In \emph{AAAI}, pages 1955--1961.

\bibitem[{Niu et~al.(2020)Niu, Zhang, Li, Cui, Liu, Li, and Zhang}]{RPJE20}
Guanglin Niu, Yongfei Zhang, Bo~Li, Peng Cui, Si~Liu, Jingyang Li, and Xiaowei
  Zhang. 2020.
\newblock Rule-guided compositional representation learning on knowledge
  graphs.
\newblock In \emph{AAAI}, pages 2950--2958.

\bibitem[{Qu and Tang(2019)}]{pLogicNet}
Meng Qu and Jian Tang. 2019.
\newblock Probabilistic logic neural networks for reasoning.
\newblock In \emph{NeurIPS}.

\bibitem[{Ray(2009)}]{HAIL}
Oliver Ray. 2009.
\newblock Nonmonotonic abductive inductive learning.
\newblock \emph{Journal of Applied Logic}, 7:329--340.

\bibitem[{Suchanek et~al.(2007)Suchanek, Kasneci, and Weikum}]{YAGO3}
Fabian~M. Suchanek, Gjergji Kasneci, and Gerhard Weikum. 2007.
\newblock {YAGO}: A core of semantic knowledge.
\newblock In \emph{WWW}, pages 697--706.

\bibitem[{Sun et~al.(2019)Sun, Deng, Nie, and Tang}]{RotatE}
Zhiqing Sun, Zhi-Hong Deng, Jian-Yun Nie, and Jian Tang. 2019.
\newblock {RotatE}: Knowledge graph embedding by relational rotation in complex
  space.
\newblock In \emph{ICLR}.

\bibitem[{Toutanova and Chen(2015)}]{FB15K237}
Kristina Toutanova and Danqi Chen. 2015.
\newblock Observed versus latent features for knowledge base and text
  inference.
\newblock In \emph{CVSC}, pages 57--66.

\bibitem[{Trouillon et~al.(2016)Trouillon, Welbl, Riedel, Éric Gaussier, and
  Bouchard}]{Trouillon:ComplEx}
Théo Trouillon, Johannes Welbl, Sebastian Riedel, Éric Gaussier, and
  Guillaume Bouchard. 2016.
\newblock Complex embeddings for simple link prediction.
\newblock In \emph{ICML}, page 2071–2080.

\bibitem[{Wang et~al.(2017)Wang, Mao, Wang, and Guo}]{TKDE:survey}
Quan Wang, Zhendong Mao, Bin Wang, and Li~Guo. 2017.
\newblock Knowledge graph embedding: A survey of approaches and applications.
\newblock \emph{IEEE Transactions on Knowledge and Dada Engineering},
  29(12):2724--2743.

\bibitem[{Wang et~al.(2019)Wang, He, Cao, Liu, and Chua}]{KGAT19}
Xiang Wang, Xiangnan He, Yixin Cao, Meng Liu, and Tat{-}Seng Chua. 2019.
\newblock {KGAT:} knowledge graph attention network for recommendation.
\newblock In \emph{{KDD}}, pages 950--958.

\bibitem[{Wang et~al.(2014)Wang, Zhang, Feng, and Chen}]{Wang:TransH}
Zhen Wang, Jianwen Zhang, Jianlin Feng, and Zheng Chen. 2014.
\newblock Knowledge graph embedding by translating on hyperplanes.
\newblock In \emph{AAAI 2014}, page 1112–1119.

\bibitem[{Xiao et~al.(2016)Xiao, Huang, and Zhu}]{ManifoldE}
Han Xiao, Minlie Huang, and Xiaoyan Zhu. 2016.
\newblock From one point to a manifold: Knowledge graph embedding for precise
  link prediction.
\newblock In \emph{IJCAI}, pages 1315--1321.

\bibitem[{Xie et~al.(2016)Xie, Liu, and Sun}]{Xie:TKRL}
Ruobing Xie, Zhiyuan Liu, and Maosong Sun. 2016.
\newblock Representation learning of knowledge graphs with hierarchical types.
\newblock In \emph{IJCAI}, page 2965–2971.

\bibitem[{Yang et~al.(2015)Yang, tau Yih, He, Gao, and Deng}]{Yang:ICLR}
Bishan Yang, Wen tau Yih, Xiaodong He, Jianfeng Gao, and Li~Deng. 2015.
\newblock Embedding entities and relations for learning and inference in
  knowledge bases.
\newblock In \emph{ICLR}.

\bibitem[{Yuan et~al.(2019)Yuan, Gao, , and Xiang}]{TransGate}
Jun Yuan, Neng Gao, , and Ji~Xiang. 2019.
\newblock {TransGate}: Knowledge graph embedding with shared gate structure.
\newblock In \emph{AAAI}.

\bibitem[{Zhang et~al.(2019)Zhang, Tay, Yao, and Liu}]{QuatE}
Shuai Zhang, Yi~Tay, Lina Yao, and Qi~Liu. 2019.
\newblock Quaternion knowledge graph embeddings.
\newblock In \emph{NeurIPS}, pages 2731--2741.

\end{thebibliography}
\bibliographystyle{acl_natbib}

\end{document}